\documentclass{article}

\usepackage[preprint]{neurips_2026} 
\usepackage{algorithm}
\usepackage{algpseudocode}
\workshoptitle{Geometric Distributional Deep Learning}

\usepackage[utf8]{inputenc} 
\usepackage[T1]{fontenc}    
\usepackage{hyperref}       
\usepackage{url}            
\usepackage{booktabs}       
\usepackage{amsfonts}       
\usepackage{amsmath}        
\usepackage{amssymb}        
\usepackage{amsthm}         
\newtheorem{proposition}{Proposition}
\newtheorem{corollary}{Corollary}

\newtheorem{conjecture}{Conjecture}
\usepackage{nicefrac}       
\usepackage{microtype}      
\usepackage{xcolor}         
\definecolor{accent}{HTML}{0E5A63}
\usepackage{tikz}\usetikzlibrary{arrows.meta,positioning,fit,backgrounds,calc}
\tikzset{
  lat/.style={circle, draw=accent, line width=0.6pt, minimum size=6mm, inner sep=0pt, fill=white},
  obs/.style={circle, draw=accent, line width=0.6pt, minimum size=6mm, inner sep=0pt, fill=accent!18},
  ctrl/.style={rectangle, draw=accent, line width=0.6pt, minimum size=6mm, inner sep=1pt, fill=accent!6},
  ed/.style={-{Stealth[length=2mm]}, line width=0.6pt, draw=accent!85},
    sw/.style  = {draw, circle, inner sep=1.2pt, font=\scriptsize, fill=accent!18, accent},
  res/.style = {draw, circle, inner sep=1.0pt, font=\scriptsize, densely dashed,
                accent!70, fill=accent!4},
  edf/.style = {-{Latex[length=1.4mm]}, accent!55, densely dashed, line width=0.5pt},
}

\newcommand{\KL}{\mathrm{KL}}

\usepackage{todonotes}

\usepackage{enumitem}

\title{The Dually Flat Geometry of Planning as Inference}

\author{
  Nikola Milosevic\\
  Neural Data Science and Statistical Computing Group\\
  Max Planck Institute for Human Cognitive and Brain Sciences\\
  Leipzig, 04103 Germany \\
  \texttt{nmilosevic@cbs.mpg.de}\\
  \And
  Asaki Kataoka\\
  Neural Computation Unit\\
  Okinawa Institute of Science and Technology\\
  Okinawa, 904-0495 Japan\\
  \texttt{asaki.kataoka@oist.jp}\\
  \And
  Nicolás Hinrichs\\
  Neural Data Science and Statistical Computing Group\\
  Max Planck Institute for Human Cognitive and Brain Sciences\\
  Leipzig, 04103 Germany\\
  \texttt{hinrichsn@cbs.mpg.de}\\
  \And
  Kenji Doya\\
  Neural Computation Unit\\
  Okinawa Institute of Science and Technology\\
  Okinawa, 904-0495 Japan\\
  \texttt{doya@oist.jp}\\
  \And
  Nico Scherf\\
  Neural Data Science and Statistical Computing Group\\
  Max Planck Institute for Human Cognitive and Brain Sciences\\
  Leipzig, 04103 Germany\\
  \texttt{nscherf@cbs.mpg.de}
}

\begin{document}

\maketitle

\begin{abstract}
  We present an alternative characterization of the occupancy measure of reinforcement learning, obtained by embedding the planning criterion into the dynamics through a \emph{resetting planning process}. Its stationary measure, which we term \emph{visitation
  measure}, is the object on which the information geometry of decision making is most naturally expressed. The achievable visitation measures form a dually flat statistical manifold whose two affine charts are the visitation probabilities and the log-policies, dual under the conditional entropy. This structure makes planning-as-inference generalize from linear rewards
  to nonlinear functionals of the visitation, each iterate solved by one natural-gradient step,
  and gives the temporal-difference error the interpretation of a marginal-utility estimate. We develop the geometry and its consequences for
  reinforcement learning and theoretical neuroscience.
\end{abstract}

\section{Introduction}
Planning as inference reformulates the maximization of return as conditioning a generative model
on the event of desirable outcomes. The reformulation places control on the same footing as
perception and learning~\cite{friston2010free} and yields principled regularizers for
control~\cite{levine2018reinforcement}. Active inference and the free energy principle extend it,
connecting planning to the minimization of expected free energy~\cite{friston2010free,da2020active}.
Yet while the free energy principle for perception is well developed, it is not settled how to
carry it over to action selection~\cite{millidge2021whence}, and a recurring difficulty is that
the planning objective and the algorithmic tools that solve it efficiently are stated in different languages. The \emph{visitation measure} of the agent's latent planning process (the stationary probability of state--action events under a policy) reconciles the two: it is at once
the variable of variational inference and the object on which dynamic programming acts.

To make this precise we construct a \emph{resetting planning process}, a latent controlled Markov chain that restarts from a fixed law at a state-action-dependent rate. Its i.i.d.\ restart cycles
give unbiased estimates of reinforcement-learning returns, and its normalized stationary measure (the visitation measure) carries the geometry we study. Our contributions are the following.

\begin{itemize}[leftmargin=*,itemsep=2pt]
    \item \textbf{A resetting characterization of the occupancy measure} (\S\ref{sec:setup}): the
  visitation measure is the stationary law of the resetting process, and its Charnes--Cooper
  transform recovers the standard occupancy LP, unifying discounted, finite-horizon,
  average-reward, and general early-termination criteria through the reset rate.
  \item \textbf{The dually flat geometry of the visitation manifold} (\S\ref{sec:md-npg}): the
  visitation measures form a dually flat statistical manifold~\cite{amari2000methods}
  under the conditional entropy, with the visitation probabilities and the log-policies as
  Legendre-dual affine charts. This gives a single geometric account of KL control as covariant
  policy search~\cite{kakade2001natural,bagnell2003covariant,peters2008natural,lan2023policy,neu2017unified,moreno2024efficient},
  in which the natural policy gradient and policy mirror descent are one update written in either
  chart~\cite{raskutti2015information}.
    \item \textbf{Variational planning beyond linear rewards} (\S\ref{sec:variational-planning}):
  dual flatness lets planning-as-inference generalize from linear rewards to nonlinear
  functionals of the visitation (convex-RL and active-inference free energies), where each
  iterate is an exact evidence lower bound solved by one natural-gradient step. The temporal-difference error
  becomes a natural-gradient estimate of the marginal utility of visitation, connecting
  dopaminergic prediction errors to stationary utility theory.
\end{itemize}

\paragraph{Related work.}
Control and planning as inference dates back to Kalman's observation that linear-quadratic control and linear filtering are computationally equivalent \cite{kalman1960new}, later termed \textit{Kalman duality}~\cite{todorov2009efficient, kappen2012optimal}. In robotics and machine learning, probabilistic reformulations of planning \cite{toussaint2009robot} and the maximum-entropy view of MDPs~\cite{ziebart2010modeling, levine2018reinforcement} underlie much of contemporary deep reinforcement learning~\cite{abdolmaleki2018maximum,haarnoja2018soft,ha2018recurrent,hafner2019dream}.

The linear occupancy program approach to Markov decision processes has its roots in the operations research literature of the 1960's~\cite{manne1960linear,d1960probleme,de1960problemes,fox1966markov,derman1970finite}. A recent line
of work beginning with maximum-entropy exploration~\cite{hazan2019provably} replaces this linear objective by a nonlinear functional of the occupancy, unified by~\cite{zhang2020variational}, and popularized by ~\cite{zahavy2021reward, geist2021concave, mutti2022challenging} as \emph{convex MDPs}. A recurring subtlety, made precise by~\cite{mutti2022importance}, is that these objectives generally differ between the finite-trial regime and the population limit. \emph{The resetting planning process} is a normalized version of state-action-dependent discounting~\cite{white2017unifying} which also has an LP formulation~\cite{jasso2024constrained}, and is precisely our transformed linear program \cite{charnes1962programming}. The resetting is the PageRank random surfer, recently connected to policy optimization by \cite{avrachenkov2026linking}. The mechanism itself is discrete-time stochastic resetting \cite{evans2020stochastic}.

The RL analogue of Amari's natural gradient~\cite{amari1998natural,amari2016information} is the natural policy gradient (NPG), introduced by~\cite{kakade2001natural}, justified information-geometrically by~\cite{bagnell2003covariant} and made an actor-critic method by~\cite{peters2008natural}. Contemporary
trust-region and proximal methods \cite{schulman2015trust,schulman2017proximal,abdolmaleki2018maximum}, are often interpreted as approximations, with connections to maximum-entropy methods~\cite{haarnoja2018soft}, made precise by a convex optimization view~\cite{neu2017unified, geist2019theory}. The dual mirror-descent viewpoint is policy mirror descent~\cite{lan2023policy, xiao2022convergence}. We explain the geometric
origin of this duality through a dual flatness perspective using the result of~\cite{raskutti2015information}. We build directly on the work of \cite{muller2024geometry}, who analyze policy gradient flows on normalized discounted occupancy measures, using the same Hessian geometry.

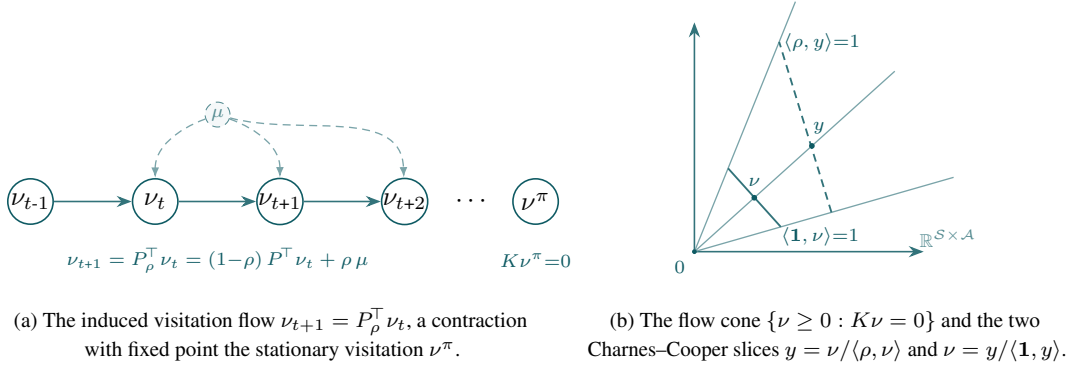
\begin{figure}[t]
\centering

\begin{minipage}[b]{0.50\linewidth}\centering
\begin{tikzpicture}[x=1.65cm]
  \node[lat] (n0) at (0,0) {$\nu_{t\text{-}1}$};
  \node[lat] (n1) at (1,0) {$\nu_{t}$};
  \node[lat] (n2) at (2,0) {$\nu_{t\text{+}1}$};
  \node[lat] (n3) at (3,0) {$\nu_{t\text{+}2}$};
  \node       (nd) at (3.55,0) {$\cdots$};
  \node[res]  (muc) at (1.5,1.15) {$\mu$};
  \foreach \a/\b in {n0/n1,n1/n2,n2/n3} {\draw[ed] (\a) -- (\b);}
  \draw[edf] (muc) to[out=200,in=90] (n1);
  \draw[edf] (muc) to[out=-20,in=90] (n2);
  \draw[edf] (muc) to[out=-15,in=92] (n3);
  \node[font=\scriptsize,text=accent] at (1.5,-0.78)
        {$\nu_{t\text{+}1}=P_\rho^{\!\top}\nu_t=(1{-}\rho)\,P^{\!\top}\nu_t+\rho\,\mu$};
  \node[lat] (fix) at (4.05,0) {$\nu^\pi$};
  \node[font=\scriptsize,text=accent] at (4.05,-0.78) {$K\nu^\pi{=}0$};
\end{tikzpicture}\\[6pt]
{\footnotesize (a) The induced visitation flow $\nu_{t+1}=P_\rho^{\!\top}\nu_t$, a contraction
with fixed point the stationary visitation $\nu^\pi$.}
\end{minipage}\hfill
\begin{minipage}[b]{0.45\linewidth}\centering
\begin{tikzpicture}[scale=0.82,>=Stealth,line join=round]
  \draw[-{Stealth[length=2mm]},line width=0.6pt,draw=accent!85] (0,0) -- (3.7,0);
  \draw[-{Stealth[length=2mm]},line width=0.6pt,draw=accent!85] (0,0) -- (0,3.7);
  \node[font=\scriptsize,text=accent!70,anchor=north east] at (4.7,0.5)
        {$\mathbb R^{\mathcal S\times\mathcal A}$};
  \foreach \a in {16,42,68} {
    \draw[accent!55,line width=0.5pt] (0,0) -- (\a:4.35);
  }
  \draw[accent!55,line width=0.5pt] (10:3.4);
  \draw[line width=0.7pt,draw=accent!85] (68:1.45) -- (16:1.45)
        node[below right=-5pt,font=\scriptsize,text=accent]{$\langle\mathbf 1,\nu\rangle{=}1$};
  \draw[densely dashed,line width=0.7pt,draw=accent!85] (16:2.3) -- (68:3.65)
        node[right=-2pt,font=\scriptsize,text=accent]{$\langle\rho,y\rangle{=}1$};
  \def\ar{42}
  \coordinate (nu) at (\ar:1.30);
  \coordinate (yy) at (\ar:2.55);

  \fill[accent] (nu) circle (1.3pt);
  \fill[accent] (yy) circle (1.3pt);
  \node[above=1pt and 0pt,font=\scriptsize,text=accent] at (nu) {$\nu$};
  \node[above right=1pt and -3pt,font=\scriptsize,text=accent] at (yy) {$y$};
  \fill[accent] (0,0) circle (1.1pt);
  \node[font=\scriptsize,below left=0pt,text=accent] at (0,0) {$0$};
\end{tikzpicture}\\[6pt]
{\footnotesize (b) The flow cone $\{\nu\ge0:K\nu=0\}$ and the two Charnes--Cooper slices $y=\nu/\langle\rho,\nu\rangle$ and $\nu=y/\langle\mathbf 1,y\rangle$.}
\end{minipage}

\caption{The restarting planning process and its visitation geometry. \textbf{(a)}~Marginalising
the reset chain for a fixed policy $\pi$ gives the visitation flow $\nu_{t+1}=P_\rho^{\!\top}\nu_t$, whose fixed point is the stationary measure
$\nu^\pi$ ($K\nu^\pi{=}0$), the object we treat as a point on a manifold and optimize in
\S\ref{sec:md-npg}. \textbf{(b)}~Relation of the standard RL liner program on occupancies $y$ and the stationary linear fractional program on $\nu$. Each ray is a set of scaling-equivalent solutions of 
$K\nu{=}0$.}
\label{fig:resetting}
\end{figure}

\section{The resetting planning processes}\label{sec:setup}

We consider the model-based planning setting illustrated in Figure \ref{fig:resetting} a), where an agent simulates a Markov process to infer policies for action selection. The agent is equipped with a latent controlled Markov process which consists of a finite state space $\mathcal{S}$, a finite action space $\mathcal{A}$, and a \emph{resetting controlled Markov kernel}
\begin{equation}
    P_\rho(s'|s,a) = (1-\rho(s,a))P(s'|s,a) + \rho(s,a)\mu(s'),
\end{equation}
where $\rho:\mathcal{S}\times\mathcal{A}\to[0,1]$ is a \emph{reset probability}, $\mu\in\mathcal{P}(\mathcal{S})$ is the start-state measure and $P$ is the agent's transition model of the environment. For constant $\rho\equiv1-\gamma$, this is the PageRank model of discounted Markov decision processes~\cite{avrachenkov2026linking} with discount factor $\gamma$. For general $\rho$ it is a normalized version of state-action-dependent discounting \cite{white2017unifying}.

Fix a memoryless stochastic policy $\pi:\mathcal{S}\to\mathcal{P}_+(\mathcal{A})$ and
define the state-action chain
$P_\rho^\pi(s',a'\mid s,a) := \pi(a'\mid s')\,P_\rho(s'\mid s,a)$. If the reset process
is unichain~\cite{puterman1994markov}, the $P_\rho^\pi$ admits a unique stationary
visitation measure $\nu^\pi\in\mathcal{P}_+(\mathcal{S}\times\mathcal{A})$ of full support,
satisfying the flow balance
\begin{equation}\label{eq:stationarity}
    \sum_a \nu^\pi(s,a) \;=\; \sum_{s',a'} P_\rho(s\mid s',a')\,\nu^\pi(s',a'),
    \qquad \forall s\in\mathcal{S}.
\end{equation}
In the following, we write the flow residual as the linear operator
\begin{equation}
    (K\nu)(s) \;=\; \sum_a \nu(s,a) \;-\; \sum_{s',a'} P_\rho(s\mid s',a')\,\nu(s',a')
\end{equation}
and assume that $\langle\rho,\nu\rangle>0$ for all $\nu\in\mathcal V_+$. The fractional-linear program
\begin{equation}\label{eq:objective}
    \sup_{\nu\in\mathcal V_+}\ \frac{\langle r,\nu\rangle}{\langle\rho,\nu\rangle},
    \qquad
    \mathcal V_+=\bigl\{\nu>0:\ K\nu=0,\ \langle\mathbf 1,\nu\rangle=1\bigr\},
\end{equation}
reproduces the standard performance criteria of reinforcement learning via its Charnes--Cooper transformation~\cite{charnes1962programming}, which yields the standard linear program of the respective Markov decision
process~\cite{manne1960linear, puterman1994markov}:
\begin{equation}\label{eq:lp}
    \sup_{y\in\mathcal Y_+}\ \langle r,y\rangle,
    \qquad
    \mathcal Y_+:=\bigl\{y>0: K y=0,\ \ \langle\rho,y\rangle=1\bigr\}.
\end{equation}
The restart probability $\rho$ then represents generalized discounting~\cite{white2017unifying} and $y$ is the standard \textit{occupancy measure} of reinforcement learning for the general discount $\gamma(s,a)=1-\rho(s,a)$, see Figure \ref{fig:resetting} b) and Appendix \ref{app:reproduce-rl}.

\paragraph{Motivation for the resetting planning process.} 
Standard discounting attaches a factor $\gamma^t$ to a non-terminating
trajectory. The resetting construction instead embeds the decision criterion into the process, which admits a natural interpretation of the occupancy measure in terms of infinite-horizon planning-as-inference as its stationary measure. With probability $\rho(s,a)$ the agent's imagined life time ends and restarts from $s_0\sim\mu$. The resulting normalization makes the same object interpretable in \textit{geometric} and \textit{inferential} terms: 

\begin{enumerate}[leftmargin=*,itemsep=2pt]
    \item \textbf{Geometrically}, a $\nu\in\mathcal V_+$ is a probability measure on $\mathcal S\times\mathcal A$, so it lives on a statistical manifold and inherits the Fisher--Rao
/ Hessian structure of information geometry. Precisely, the feasible set in Eq. \ref{eq:objective} is the intersection of the open $(\mathcal S\times\mathcal A)$-simplex with the affine
subspace $\ker K$, see Figure~\ref{fig:geom}. Since the $|\mathcal{S}|$ rows of $K$ have
rank $|\mathcal{S}|-1$ and $\mathcal{V}_+$ is relatively open, it is a smooth submanifold of the open
$(\mathcal S\times\mathcal A)$-simplex of dimension $|\mathcal{S}|(|\mathcal{A}|-1)$. We will refer to $\mathcal V_+$ as the \emph{visitation manifold} of the planning process.
     \item \textbf{Inferentially}, the stationary law of the restarting process is a variational posterior where planning
  becomes conditioning a generative model on optimality, with $\nu$ the posterior over
  state--action events and the free energy its evidence bound (\S\ref{sec:planning-inference}). With restarts, a single episode is an i.i.d.\ trajectory sample of a.s. finite length, and $\nu$ is the stationary measure of the process. Generalizing to nonlinear free energy functionals $F(\nu)$ is then meaningful from the point of view of variational inference and the descending directions can be estimated as policy advantages from a finite number of finite-length trajectories~\cite{agarwal2021theory,schulman2015high,mutti2022challenging} without differentiation through the agent's model.
\end{enumerate}
The visitation measure is thus the natural variable in which the information geometry, the control-as-inference aspect, and the advantage estimator all can be jointly analyzed.

\section{Dually flat structure of the visitation manifold}

A point of the \emph{visitation manifold} is a stationary visitation measure $\nu\in\mathcal V_+$. In the ambient linear
coordinates, $\mathcal V_+$ is a relatively open convex subset of $\ker K$, so these restrict to a
global affine chart. The corresponding non-redundant \emph{mixture chart}
$\eta(\nu)\in\mathbb R^{|\mathcal S|(|\mathcal A|-1)}$ is constructed in
Appendix~\ref{app:mixture-chart}. We use $\nu$ and its ambient representation interchangeably for
mixtures, and reserve $\eta$ for statements needing a homeomorphism to
$\mathbb R^{|\mathcal S|(|\mathcal A|-1)}$. The construction is visualized in Figure \ref{fig:geom}.

\begin{figure}[t]
    \centering
    \begin{tikzpicture}[
    >={Stealth[length=2mm]},
    open/.style   ={draw=accent,dashed,line width=1pt,line join=round},
    shadow/.style ={draw=accent!55,dotted,line width=0.9pt,line join=round},
    ambient/.style={draw=accent,dashed,line width=1pt,line join=round},
    ax/.style     ={draw=accent!70,line width=0.8pt,->},
    tang/.style   ={draw=accent,line width=1pt,-{Stealth[length=1.5mm]}},
    map/.style    ={draw=accent,line width=1.1pt,->},
    lbl/.style    ={font=\small,text=accent!85},
    sub/.style    ={font=\scriptsize,text=accent!85},
    capt/.style   ={font=\footnotesize,align=center,text=accent!85},
    perp/.style   ={draw=accent!55,line width=0.9pt,densely dashed},
  ]

  \begin{scope}[shift={(1.55,5.15)},rotate=-18]
    \shade[inner color=accent!14,outer color=accent!24] (-1.05,-1.05) rectangle (1.05,1.05);
    \draw[open] (-1.05,-1.05) rectangle (1.05,1.05);
  \end{scope}

  \begin{scope}[shift={(1.55,5.15)}]
    \draw[tang] (0,0) -- (0.72,0.20);
    \draw[tang] (0,0) -- (-0.10,0.70);
    \fill[accent] (0,0) circle (1.5pt);
    \node[sub,anchor=north] at (0.02,-0.06) {$\nu$};
    \node[sub,anchor=west] at (0.70,0.24) {$\partial_i$};
  \end{scope}
  \node[lbl] at (1.4,3.72) {$\mathcal V_{+}$};

    \shade[inner color=accent!10,outer color=accent!22]
        (9.10,6.25) -- (7.60,4.20) -- (10.70,4.20) -- cycle;
  \draw[ambient] (9.10,6.25) -- (7.60,4.20) -- (10.70,4.20) -- cycle;

  \coordinate (nu) at (8.95,5.05);

  \draw[perp] ($(nu)+(-1.55,-0.586)$) -- ($(nu)+(1.75,0.661)$);
  \node[sub,anchor=west]  at ($(nu)+(1.78,0.672)$) {$\ker K$};

  \draw[perp]  ($(nu)+(0.42,-1.11)$) -- ($(nu)+(-0.60,1.59)$);
  \node[sub,anchor=south] at ($(nu)+(-0.62,1.63)$) {$\operatorname{im}K^{\!\top}$};

  \draw[accent!55,line width=0.6pt]
        ($(nu)+(0.18,0.068)$)
     -- ($(nu)+(0.18,0.068)+(-0.068,0.18)$)
     -- ($(nu)+(-0.068,0.18)$);

  \draw[tang] (nu) -- ($(nu)+(0.45,0.170)$);
  \node[sub,anchor=north west,xshift=0pt,yshift=0pt] at ($(nu)+(0.4,0.2)$) {$\partial_i$};

  \fill[accent] (nu) circle (1.6pt);
  \node[sub,anchor=north east,xshift=1pt,yshift=-1pt] at (nu) {$\nu$};
  \node[sub,anchor=north] at (9.10,4.10) {$\mathcal P_+(\mathcal S\times\mathcal A)$};

  \draw[map] (3.15,5.35) to[bend left=7] (7.4,5.35);
  \node[lbl,anchor=south] at (5.15,5.60) {$\iota$};

  \coordinate (O1) at (0.50,0.70);
  \draw[ax] (O1) -- ++(3.15,0);
  \draw[ax] (O1) -- ++(0,2.25);
  \draw[shadow] (0.50,0.70) -- (3.05,0.70) -- (0.50,2.65) -- cycle;
  \shade[inner color=accent!14,outer color=accent!24]
      (0.80,0.95) -- (2.30,0.95) -- (1.80,1.58) -- (0.80,1.58) -- cycle;
  \draw[open]
      (0.80,0.95) -- (2.30,0.95) -- (1.80,1.58) -- (0.80,1.58) -- cycle;
  \node[capt,anchor=north] at (1.85,0.42)
      {visitation probabilities $\eta(\nu)$\\[-1pt] {\scriptsize mixture (m) chart}};

  \coordinate (O2) at (7.90,0.70);
  \draw[ax] (O2) -- ++(3.20,0);
  \draw[ax] (O2) -- ++(0,2.25);
  \shade[inner color=accent!14,outer color=accent!24] (8.15,0.95) rectangle (10.55,2.40);
  \draw[open] (8.15,0.95) rectangle (10.55,2.40);
  \node[capt,anchor=north] at (9.35,0.42)
      {log-probabilities of the policy $\theta(\nu)$\\[-1pt] {\scriptsize exponential (e) chart}};

  \node[lbl] at (5.25,2.05) {$\mathbb{R}^{\,|\mathcal S|(|\mathcal A|-1)}$};
  \draw[map,{Stealth[length=1.8mm]}-{Stealth[length=1.8mm]}]
      (3.75,1.45) to[bend right=15] (6.75,1.45);
  \node[capt] at (5.35,0.80) {$\log\pi=\nabla_{\!\mathcal V_{+}}\varphi(\nu)$};

  \node[font=\small] at (-0.5,6.5) {a)};
  \node[font=\small] at (-0.5,3.25) {b)};

\end{tikzpicture}
    \caption{The visitation geometry in one picture. a) the manifold via its inclusion $\iota(\mathcal{V}_+)=\mathrm{ker}K\cap\mathcal P_+(\mathcal S \times \mathcal A)$
    in the state-action simplex. b) $\mathcal V_+$ admits two global affine charts, the
    mixture coordinate $\eta$ of the visitation probabilities (left) and the exponential
    coordinate $\theta$ of the log-policies (right).}
    \label{fig:geom}
\end{figure}
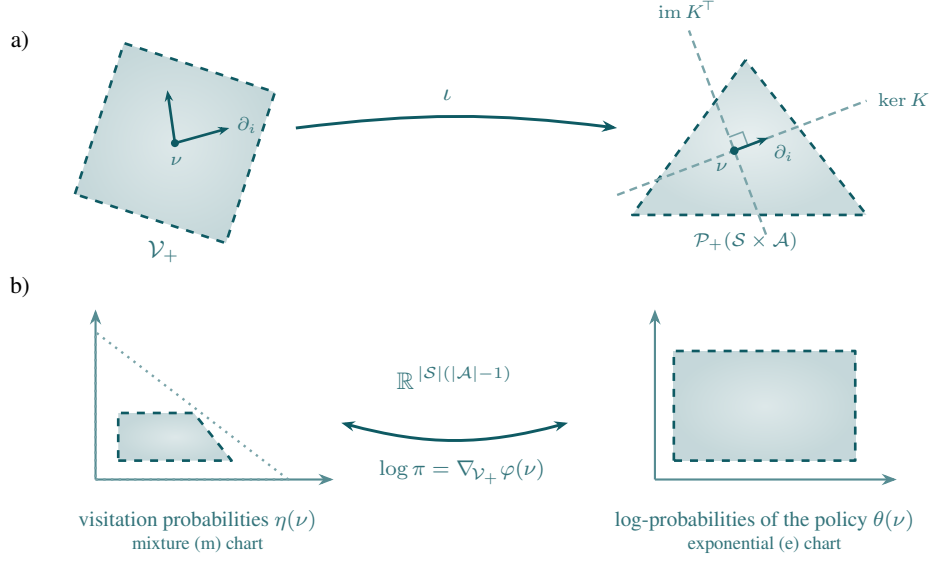

\paragraph{Visitations and policies are diffeomorphic.} The maps~\eqref{eq:cc-maps} are mutually inverse real-analytic diffeomorphisms between the mass section $\mathcal V_+=\{\langle\mathbf 1,\nu\rangle=1\}$ and the reset-rate section
$\{\langle\rho,y\rangle=1\}$ of the flow cone $\{z\ge0:Kz=0\}$, well defined since
$\langle\rho,\nu\rangle>0$. 
Every positive measure $\nu\in\mathcal{P}_+(\mathcal{S}\times\mathcal{A})$
satisfying \eqref{eq:stationarity} is the stationary measure of the memoryless policy
recovered by conditioning, $\pi_\nu(a\mid s)=\nu(s,a)/\!\sum_{a'}\nu(s,a')$, and conditioning is scale-invariant, $\pi_y=\pi_\nu$, so
$\nu\mapsto\nu/\langle\rho,\nu\rangle$ fixes the induced policy and moves only the normalization. The conditioning map $y\mapsto\pi_y$ is a diffeomorphism~\cite{muller2024geometry}, hence the same is true for $\nu\mapsto\pi_\nu$. 

\paragraph{Dual affine coordinates.} The negative conditional entropy
$\varphi:\mathcal V_+\to\mathbb R\cup\{+\infty\}$,
\begin{equation}\label{eq:phi}
    \varphi(\nu)=\sum_{s,a}\nu(s,a)\,\log\frac{\nu(s,a)}{\sum_{a'}\nu(s,a')},
\end{equation}
is convex on $\mathcal V_+$ (conditional entropy is concave in the joint law) and generates the
statistical structure. Its differential is the \emph{exponential coordinate}
\begin{equation}\label{eq:dual-coord}
    \theta(\nu):=\nabla\varphi(\nu),\qquad
    \frac{\partial\varphi}{\partial\nu(s,a)}=\log\pi_\nu(a\mid s),
\end{equation}
the log-policy modulo state-wise constants $c_s$. Thus the two affine charts of the same point
$\nu$ are the mixture coordinate $\eta$ (the visitation, m-flat) and the exponential coordinate
$\theta$ (the log-policy, e-flat), Legendre-dual through $\theta=\nabla\varphi(\eta)$ and
$\eta=\nabla\varphi^\ast(\theta)$. Appendix~\ref{app:mixture-chart} derives this pairing on the
explicit charts. Note that the reset rate deforms the m-chart and the Hessian geometry, but not the e-chart. This is the first form of the policy--visitation duality: \emph{a visitation and its log-policy are the mixture and exponential coordinates of one point.}

\paragraph{Tangent space.} At $\nu\in\mathcal V_+$ the \emph{tangent space} is
$T_\nu\mathcal V_+=\ker K\cap\ker\mathbf 1^{\!\top}$, of dimension $|\mathcal S|(|\mathcal A|-1)$.
Every $u\in T_\nu\mathcal V_+$ decomposes uniquely into a marginal and a conditional part,
\begin{equation}\label{eq:tangent-decomp}
    u(s,a)=u_S(s)\,\pi_\nu(a\mid s)+\nu_S(s)\,\dot\pi_u(a\mid s),
    \quad \textstyle\sum_a\dot\pi_u(a\mid s)=0,\ \ u_S(s):=\textstyle\sum_a u(s,a),
\end{equation}
and $u\mapsto\dot\pi_u$ is a linear isomorphism onto the tangent space of the product action
simplex, a \emph{policy variation}. Moreover, the marginal part is fully determined by $\dot\pi_u$ through the constraint $Ku=0$, more precisely its linearized flow $u_S(s)=\sum_{s',a'}P_\rho(s\mid s',a')u(s',a')$. So a
\emph{visitation variation} $u$ (m-side) and its \emph{policy variation} $\dot\pi_u$ (e-side) are two representations of one tangent vector, which is the infinitesimal form of the chart duality above.

\paragraph{Cotangent space.} Dually, the \emph{cotangent space} is
$T^\ast_\nu\mathcal V_+=(\mathbb R^{\mathcal S\times\mathcal A})^\ast/(\operatorname{im}K^{\!\top}+\mathbb R\mathbf 1)$,
with $(K^{\!\top}\lambda)(s,a)=\lambda(s)-\sum_{s'}P_\rho(s'\mid s,a)\lambda(s')$. A reward function $r$ and its shifts $r+K^{\!\top}\lambda+c\mathbf 1$ represent the same cotangent vector,
so covectors are state-action functions modulo (resetting-based) potential shaping ($\operatorname{im}K^{\!\top}$~\cite{ng1999policy}) and constants ($\mathbb R\mathbf 1$), exactly the symmetries of the advantage function~\cite{schulman2015high}. They act as linear forms on tangent vectors on $\mathcal{V}_+$. 

The canonical representative of the cotangent class $[r]_\nu$ is the centered \emph{advantage}
$A_\nu[r]$, obtained by choosing $\lambda=V_\nu[r]$, the value of the resetting chain. The value is defined by the resetting Bellman (Poisson) equation \emph{with
its gain term},
\begin{equation}\label{eq:value}
    V_\nu[r](s)+J(\nu)\,\bar\rho(s)
    =\mathbb E_{\pi_\nu(\cdot\mid s)}\!\Big[r(s,a)+\textstyle\sum_{s'}P_\rho(s'\mid s,a)\,V_\nu[r](s')\Big],
    \quad J(\nu)=\frac{\langle r,\nu\rangle}{\langle\rho,\nu\rangle},
\end{equation}
with $\bar\rho(s)=\mathbb E_{\pi_\nu(\cdot|s)}[\rho(s,a)]$, which determines $V_\nu[r]$ uniquely up to the constant gauge
$\mathbb R\mathbf 1$. The advantage is then the reduced cost
\begin{equation}\label{eq:advantage}
    A_\nu[r]:=r-J(\nu)\,\rho-K^{\!\top}V_\nu[r],
    \qquad \mathbb E_{\pi_\nu(\cdot\mid s)}\big[A_\nu[r](s,\cdot)\big]=0,
\end{equation}
which is invariant to the constant gauge in $V_\nu[r]$, so it is well defined as the centered
representative of $[r]_\nu$ even though $V_\nu[r]$ is not.\footnote{%
For constant $\rho\equiv1-\gamma$, $\bar\rho\equiv1-\gamma$ and \eqref{eq:advantage} reduces to the
ordinary discounted advantage $r+\gamma P V_\nu[r]-\mathbb E_{\pi_\nu}[\cdot]$: the gain term
$J(\nu)\rho=(1-\gamma)J(\nu)\mathbf 1$ is a global constant and cancels in the centering.}
The advantage, read as the linear form $\langle A_\nu[r],\cdot\rangle$ on $T_\nu\mathcal V_+$, is
the cotangent object dual to visitation variations: for $u\in T_\nu\mathcal V_+$ the pairing
$\langle A_\nu[r],u\rangle=\langle[r]_\nu,u\rangle=\mathrm dJ_\nu(u)$ is the directional derivative
of the return, gauge-invariant because $u\in\ker K\cap\ker\mathbf 1^{\!\top}$ annihilates both
$\operatorname{im}K^{\!\top}$ and $\mathbb R\mathbf 1$.

\paragraph{Metric.} The Hessian $g_\nu(u,v)=\nabla^2\varphi(\nu)(u,v)$ is
\begin{equation}\label{eq:metric}
    g_\nu(u,v)=\sum_{s,a}\frac{u(s,a)\,v(s,a)}{\nu(s,a)}-\sum_s\frac{u_S(s)\,v_S(s)}{\nu_S(s)}
    =\sum_s\nu_S(s)\sum_a\frac{\dot\pi_u(a\mid s)\,\dot\pi_v(a\mid s)}{\pi_\nu(a\mid s)},
\end{equation}
the ambient joint Fisher--Rao metric \emph{minus} its state-marginal part, equivalently the visitation-weighted sum of per-state action-simplex Fisher-Rao metrics. This is the state-action Fisher metric of the natural policy gradient~\cite{kakade2001natural}, also called the \textit{Kakade metric}~\cite{muller2024geometry}. It is not
the joint Fisher--Rao of $\nu$, which does not respect the flow constraint. The Kakade metric is positive \emph{semi}definite on the ambient space (degenerate along policy-preserving marginal changes) but
positive definite on $T_\nu\mathcal V_+$, where those directions are excluded by $Ku=0$.

\paragraph{Dual flatness and mirror descent.}\label{sec:md-npg}
The two charts each carry a flat connection. The mixture connection $\nabla^{\mathrm m}$, for which
$\nu$ is affine, has mixture geodesics $t\mapsto(1-t)\nu_0+t\nu_1$. The exponential connection
$\nabla^{\mathrm e}$, for which $\theta=\log\pi_\nu$ is affine, has log-linear geodesics
$t\mapsto\pi_0^{1-t}\pi_1^{t}$ (normalized). They are mutually dual w.r.t.\ $g$, so
$(\mathcal V_+,g,\nabla^{\mathrm m},\nabla^{\mathrm e})$ is a dually flat statistical
manifold~\cite{amari2016information,shima2007geometry}, with canonical divergence the
visitation-weighted conditional KL,
\begin{equation}\label{eq:bregman}
    D_\varphi(\nu\|\nu')=\sum_s\nu_S(s)\,\KL\big(\pi_\nu(\cdot\mid s)\,\|\,\pi_{\nu'}(\cdot\mid s)\big),
\end{equation}
the trust region of on-policy optimization. The m- and e-geodesics satisfy the generalized Pythagorean theorem (Appendix~\ref{app:pythagoras}).

Consequently mirror descent and the natural policy gradient are one update in the two charts.
With $\varphi$ as mirror map, the proximal step
\begin{equation}\label{eq:pmd}
    \nu_{k+1}=\arg\min_{\nu\in\mathcal V_+}\ \langle\nabla_\nu f(\nu_k),\nu\rangle
    +\tfrac1\alpha D_\varphi(\nu\|\nu_k)
\end{equation}
is, by \cite[Thm.~1]{raskutti2015information}, natural-gradient descent in the exponential chart,
and the two coincide for every $\alpha>0$. To first order in $\alpha$ they agree with the primal
natural-gradient step $\nu_k-\alpha\,g(\nu_k)^{-1}\nabla_\nu f(\nu_k)$, and all three coincide with
the common Hessian gradient flow $\dot\nu=-\mathrm{grad}_g f$ as $\alpha\to0$~\cite{alvarez2004hessian}.
Its first-order condition is an additive step in the exponential coordinate. For
$f(\nu)=-\langle r,\nu\rangle$ the increment is the advantage,
\begin{equation}\label{eq:advantage-update}
 \log\pi_{k+1}=\log\pi_k+\alpha\,A_{\nu_{k+1}}[r]\ \ (\mathrm{mod}\ \text{values}),
\end{equation}
so \emph{the advantage is the exponential-coordinate increment}. The update is implicit: the
increment carries the advantage at the \emph{updated} iterate $\nu_{k+1}$, through the value gauge
that centers it. Freezing it to the current iterate, $A_{\nu_k}[r]$, gives the explicit
exponentiated-advantage step $\pi_{k+1}\propto\pi_k\exp(\alpha A_{\nu_k}[r])$, the first-order
approximation whose error is well understood~\cite{lan2023policy,xiao2022convergence}. We proceed
with the idealized implicit update.

\paragraph{Infinite-dimesnsional and Wasserstein geometries.} A full measure-theoretic extension to continuous
$\mathcal S,\mathcal A$ in the sense of \cite{ay2017information} is an open challenge, but exponential
families are known to form finite-dimensional dually flat submanifolds of the space of measures~\cite{amari2000methods}. The analogue here is exponential-family policies in linearly parameterized
MDPs~\cite{jin2020provably}, developed in Appendix~\ref{app:linear-mdp}.
$\mathcal{V}_+$ can also be equipped with an $L^2$-Wasserstein structure in continuous or discrete $\mathcal{S}\times\mathcal A$~\cite{li2018natural} derived from the sample space's own geometry. Unlike our Fisher--Rao/Hessian geometry, which reweights probability in place, Wasserstein geodesics transport visitation mass across states. The two are complementary natural gradients on the same polytope.

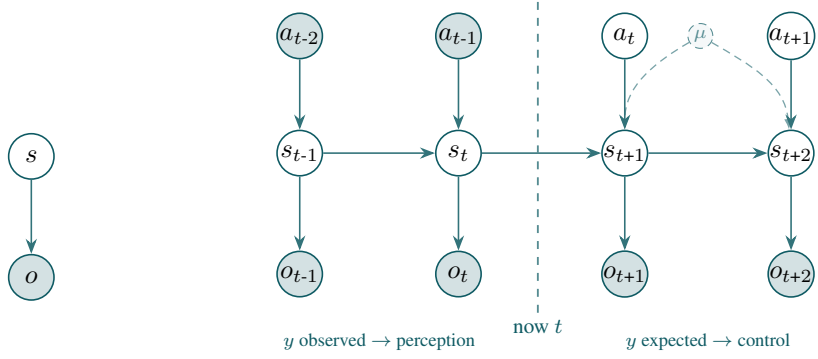
\begin{figure}[t]\centering
\centering
\begin{minipage}[b]{0.28\linewidth}\centering
\begin{tikzpicture}

  \node[lat] (x1) at (0,0)   {$s$};
  \node[obs] (y)  at (0,-1.6){$o$};
  \draw[ed] (x1) -- (y);
\end{tikzpicture}\\[28pt]
{\footnotesize (a) Perception: infer $q(x\mid y)$.}
\end{minipage}
\begin{minipage}[b]{0.68\linewidth}\centering
\begin{tikzpicture}

  \node[lat] (xa) at (0,0)   {$s_{t\text{-}1}$};
  \node[lat] (xb) at (2.1,0) {$s_{t}$};
  \node[lat] (xc) at (4.3,0) {$s_{t\text{+}1}$};
  \node[lat] (xd) at (6.5,0) {$s_{t\text{+}2}$};

  \draw[ed] (xa) -- (xb); \draw[ed] (xb) -- (xc); \draw[ed] (xc) -- (xd);

  \node[res] (mu) at (5.3,1.55) {$\mu$};

  \draw[edf] (mu) to[out=-150,in=90]  (xc);
  \draw[edf] (mu) to[out=-30,in=95]  (xd);

  \node[obs] (ya) at (0,-1.6)   {$o_{t\text{-}1}$};
  \node[obs] (yb) at (2.1,-1.6) {$o_{t}$};
  \draw[ed] (xa) -- (ya); \draw[ed] (xb) -- (yb);

  \node[obs] (ua) at (0,1.55)   {$a_{t\text{-}2}$};
  \node[obs] (ub) at (2.1,1.55) {$a_{t\text{-}1}$};
  \draw[ed] (ua) -- (xa); \draw[ed] (ub) -- (xb);

  \node[lat] (uc) at (4.3,1.55) {$a_{t}$};
  \node[lat] (ud) at (6.5,1.55) {$a_{t\text{+}1}$};
  \draw[ed] (uc) -- (xc); \draw[ed] (ud) -- (xd);

  \node[obs] (Oa) at (4.3,-1.6) {$o_{t\text{+}1}$};
  \node[obs] (Ob) at (6.5,-1.6) {$o_{t\text{+}2}$};
  \draw[ed] (xc) -- (Oa); \draw[ed] (xd) -- (Ob);

  \draw[dashed,accent!70,line width=0.7pt] (3.15,2.0) -- (3.15,-2.2);
  \node[accent,font=\footnotesize] at (3.15,-2.3) {now $t$};

  \node[font=\scriptsize,text=accent] at (1.05,-2.5) {$y$ observed $\to$ perception};
  \node[font=\scriptsize,text=accent] at (5.40,-2.5) {$y$ expected $\to$ control};
\end{tikzpicture}\\[6pt]
{\footnotesize (b) Planning as inference on the resetting chain.}
\end{minipage}
\caption{Planning as inference with restarting state visitation. (a) standard inferential perception infers latent state $s$ from observation $o$; (b) the
resetting generative model with restart measure $\mu$. At each decision $(s_t,a_t)$, with prob.\ $1{-}\rho(s_t,a_t)$
the state continues under $P$, with prob.\ $\rho$ it restarts from $\mu$. Planning is concerned with inferring the latent future $a_t$, e.g. as an open-loop sequence.}
\label{fig:pai}
\end{figure}

\section{Variational planning beyond linear rewards}
\label{sec:variational-planning}

Planning as inference casts control as maximizing an evidence lower bound (ELBO) whose reward
term is \emph{linear} in the visitation $\nu\in\mathcal V_+$: variational inference on trajectory
distributions with a structured posterior and exponential likelihood yields the standard soft-RL
formulation, which we construct directly on the stationary measure of the restarting chain
(Appendix~\ref{app:planning-as-inference}). Section~\ref{sec:nonlinear} asks what happens
when the objective is instead a \emph{nonlinear} stationary free energy, inducing a general-utility
MDP~\cite{zhang2020variational,zahavy2021reward}. This arises whenever a state-entropy term enters
the free energy, which happens in maximum-state-entropy exploration~\cite{hazan2019provably,ashlag2026probing},
active inference~\cite{milosevic2026active}, risk-averse
objectives~\cite{mutti2022challenging}, and imitation learning~\cite{ho2016generative}. We show
the dually flat geometry of $\mathcal V_+$ allows extending the soft-RL view to this case, reducing the nonlinear problem to a sequence of linear
entropy-regularized MDPs.

\subsection{The linear stationary free energy}
\label{sec:planning-inference}
Fix a reference policy $\pi_{\mathrm{ref}}\in\Pi_+$ with stationary visitation
$\nu^{\mathrm{ref}}\in\mathcal V_+$. A candidate $\pi$ is the future-action posterior and induces
$\nu\in\mathcal V_+$ (Figure~\ref{fig:pai}), both stationary laws of the same resetting chain.
Given a reward $r$ and preference likelihood $p(o{=}o^\star\mid s,a)\propto\exp(\beta^{-1}r(s,a))$,
$0<\beta<\infty$, define the stationary free energy
\begin{equation}\label{eq:linear-fe}
  \mathcal F(\nu):=\langle r,\nu\rangle-\beta\,D_\varphi(\nu\,\|\,\nu^{\mathrm{ref}}),
  \quad
  D_\varphi(\nu\|\nu^{\mathrm{ref}})=\sum_s\nu_S(s)\,\KL\big(\pi_\nu(\cdot\mid s)\,\|\,\pi_{\mathrm{ref}}(\cdot\mid s)\big),
\end{equation}
whose information cost is the manifold's own canonical divergence (\S\ref{sec:md-npg}). Since
$\nu^{\mathrm{ref}}$ has full support, the convex conjugate of $D_\varphi(\cdot\|\nu^{\mathrm{ref}})$
at $r/\beta$,
\begin{equation}\label{eq:logZ}
  \log Z:=\sup_{\nu\in\mathcal V_+}\big\{\langle r/\beta,\nu\rangle-D_\varphi(\nu\|\nu^{\mathrm{ref}})\big\},
\end{equation}
is finite, and Fenchel--Young gives the evidence lower bound, stated entirely in the visitation
measure,
\begin{equation}\label{eq:direct-elbo}
  \beta^{-1}\mathcal F(\nu)\ \le\ \log Z,
  \quad
  \log Z-\beta^{-1}\mathcal F(\nu)=D_\varphi(\nu\,\|\,\nu^\star),\quad \nu^\star=\arg\max_{\mathcal V_+}\mathcal F .
\end{equation}
The bound tightens as $\nu\to\nu^\star$, with slack the Bregman divergence to the optimum. $\log Z$
is a log-partition function, equal to the path-space log-evidence of the resetting generative model
(Appendix~\ref{app:pi-elbo}).

The maximizer is a soft Bellman fixed point: stationarity of \eqref{eq:logZ} gives the
exponential-coordinate tilt $\log(\pi^\star/\pi_{\mathrm{ref}})=\beta^{-1}A_{\nu^\star}[r]$ (mod
values), with the value the fixed point of the log-sum-exp backup through the resetting kernel
$P_\rho$. For $\rho\equiv1-\gamma$ the reset branch $\rho\langle\mu,V^\star\rangle$ is a global
constant that cancels in the advantage, so \eqref{eq:direct-elbo} reduces to textbook soft value
iteration, $\pi^\star\propto\pi_{\mathrm{ref}}\exp(A_{\nu^\star}[r]/\beta)$: \eqref{eq:linear-fe}
is an entropy-regularized MDP~\cite{geist2019theory} and standard soft-RL
methods~\cite{levine2018reinforcement,haarnoja2018soft,abdolmaleki2018maximum} apply with $r$ as
reward. The dually flat structure recasts this in geometric terms and, next, generalizes it to
nonlinear free energies.

\subsection{Nonlinear free energies and their solution}
\label{sec:nonlinear}

The linear free energy \eqref{eq:linear-fe} is the special case where the information cost is the
manifold's own divergence. More generally, \emph{free-energy planning} maximizes any smooth
concave functional of the visitation,
\begin{equation}\label{eq:general-fe}
  \max_{\nu\in\mathcal V_+}\ -F(\nu),\qquad F\ \text{convex on}\ \mathcal V_+,
\end{equation}
the general-utility / convex-MDP regime~\cite{zhang2020variational,zahavy2021reward}. The objectives
of interest are instances, differing only in $\nabla F$:

\emph{(i) Biased generative models.} Scoring states against a preferred distribution
$\tilde p\in\mathcal P_+(\mathcal S)$ with base measure $m(s,a)\propto\tilde p(s)\pi_{\mathrm{ref}}(a\mid s)$
gives $-F(\nu)=\langle\beta^{-1}r,\nu\rangle-\KL(\nu\|m)$, a biased generative model in the language
of active inference~\cite{da2020active}. Splitting $\KL(\nu\|m)$ into its state and conditional
parts,
\begin{equation}\label{eq:fe-expanded}
  -F(\nu)=\langle r_{\mathrm{lin}},\nu\rangle-\varphi(\nu)+H(\nu_S),
  \qquad r_{\mathrm{lin}}:=\beta^{-1}r+\log\tilde p+\log\pi_{\mathrm{ref}},
\end{equation}
whose only nonlinear term is the state-marginal entropy $H(\nu_S)$, the maximum-state-entropy
exploration objective of~\cite{hazan2019provably}.

\emph{(ii) Risk-sensitive objectives.} A coherent or static risk functional of the sequential
outcome, e.g.\ a CVaR-type measure, is likewise concave in
$\nu$~\cite[Table~1]{mutti2022challenging}. The nonlinearity appears in the preference rather than an
exploration bonus, but the solution below is unchanged.

Since $-F$ is concave, \eqref{eq:general-fe} is a convex MDP and the soft Bellman picture of
\S\ref{sec:planning-inference}, which needed $F$ linear in $\nu$, no longer applies directly. Dual
flatness resolves this by mirror descent, which needs only the linearization of $F$ at the current
iterate. Its ambient Euclidean gradient
$-\nabla F(\nu_k)$ fixes a covector representative of the differential $-\delta F[\nu_k]/\delta\nu$,
and the mirror step 
\begin{equation}\label{eq:convex-md}
    \nu_{k+1}=\arg\max_{\nu\in\mathcal V_+}\ \langle -\nabla F(\nu_k),\,\nu\rangle
    -\tfrac1\alpha D_\varphi(\nu\|\nu_k)
\end{equation}
is the per-iterate regularized MDP with that reward and advantage update
\begin{equation}\label{eq:convex-md-advantage-update}
    \log\pi_{k+1}=\log\pi_k+\alpha\,A_{\nu_{k+1}}\!\big[-\nabla F(\nu_k)\big]\ (\mathrm{mod\ values}).
\end{equation}
The increment is the advantage of the frozen linearized reward, evaluated at the updated iterate. 

Convex free-energy planning is thus variational EM~\cite{aubin2022mirror}, the E-step the
linear regularized MDP \eqref{eq:convex-md} (\cite[MD-MPI]{geist2019theory}) and the M-step
the relinearization $\nabla F(\nu_{k+1})$. The implicit update \eqref{eq:convex-md-advantage-update} is obtained by solving the soft-Bellman fixed point \eqref{eq:soft-bellman-planning} with frozen reward $-\nabla F$, reference $\pi_k$, and temperature $\alpha^{-1}$. Its self-consistent value yields the next-iterate advantage $A_{\nu_{k+1}}$. Note that $\pi_k$ is the inner reference with temperature $\alpha$ and $\pi_\mathrm{ref}$ is the outer reference with temperature $\beta$. For (i), the intrinsic reward is the marginal utility of
a stationary utility~\cite{koopmans1960stationary}, $-\nabla F(\nu)=r_{\mathrm{lin}}-\log\nu(s,a)-1$
(using $-\log\nu_S-\log\pi_\nu=-\log\nu$), a policy-dependent novelty bonus on under-visited
preferred states~\cite{zahavy2021reward}. Assigning an ambient linear form to a differential on
$\mathcal V_+$ is the Fisher-cometric construction of~\cite{nagaoka2024fisher}.\footnote{Their
target is the open simplex, with cotangent quotient $\mathbb R^{\mathcal S\times\mathcal A}/\mathbb R\mathbf 1$. Ours carries the flow gauge as well, $\mathbb R^{\mathcal S\times\mathcal A}/(\operatorname{im}K^\top+\mathbb R\mathbf 1)$.}
The reset structure makes each per-iterate advantage estimable without backpropagation through the
model, with tunable bias--variance~\cite{schulman2015high,agarwal2021theory}.

The policy sequence is the minorize--maximize envelope of $F$, as in deep actor-critics~\cite{schulman2015trust},
except that function approximation breaks dual flatness and the global-convergence guarantee. In the
tabular and linear-MDP cases (Appendix~\ref{app:linear-mdp}) dual flatness supplies it whenever
$(F,\varphi)$ are relatively smooth: for $-\nabla^2F\preceq L_F\,g$ on $T\mathcal V_+$, mirror ascent
with step $1/L_F$ converges at rate $L_F/k$ in $D_\varphi$~\cite{lu2018relatively}, and by
\S\ref{sec:md-npg} so does its NPG form. Relative smoothness must be checked per objective. We
believe it holds broadly, since the generic obstruction is the state entropy, whose $1/\nu_S$
curvature is matched by $g$ (\cite{milosevic2026active}, finite horizon; Appendix~\ref{app:relsmooth},
discounted).

\subsection{Application in theoretical neuroscience}\label{sec:neuro}

\paragraph{The origin of reward prediction error.}
Decision making in the brain has long been understood through reinforcement learning and control as
inference, most directly through the reward prediction error (RPE) encoded by phasic
dopamine~\cite{schultz1997neural,schultz2016dopamine}. A key refinement is that these errors are
computed in units not of delivered reward but of a nonlinear \emph{utility} of reward magnitude,
recovered from risky-choice behaviour~\cite{stauffer2014dopamine,alikaya2018reward} and later
placed within distributional RL~\cite{dabney2020distributional}. For sequential decisions, distributional utility measures and
nonlinear occupancy programs can coincide: a coherent risk functional, e.g.\ a CVaR-type measure
via its risk envelope, is a concave functional of the occupancy measure~\cite[Table~1]{mutti2022challenging},
hence a convex MDP in the sense of \S\ref{sec:variational-planning}.

This gives a free-energy interpretation to the utility RPE through Koopmans' stationary utility
theory~\cite{koopmans1960stationary}. A stationary recursive utility takes the aggregator form
$F=V(r,\delta F)$ relating immediate to prospective utility, but leaves $V$ underdetermined. Our
construction fixes it as the soft Bellman backup, the aggregator induced by the free-energy view. It is time-consistent because it is a stationary function of
state. By this interpretation, the quantity dopamine mediates is the marginal utility of the stationary free energy,
$r_k=-\nabla F(\nu)$, whose reward-magnitude slope is the marginal utility $u'$
recovered behaviourally~\cite{stauffer2014dopamine,alikaya2018reward}. The geometry predicts that the RPE is an estimate of the advantage that belongs to a shaped intrinsic reward, hence state- and
policy-dependent rather than a fixed prediction target. This is a refinement that is testable in principle. Note however that this is
an interpretive correspondence at the level of the signal an optimal risk-sensitive agent would compute, not a claim about neural implementation.

\paragraph{Perception--action duality in the cortex.}
A parallel question is how control is realized in the cortex. Motivated by the shared laminar
architecture of sensory and motor areas and by the classical filtering--control
duality~\cite{kalman1960new,todorov2009efficient}, \cite{doya_canonical_2021} maps inference and
control variables onto (thalamo-)cortical circuitry, leaving the precise correspondence open. The
information geometry suggests a sharper form of the question. A recent predictive-coding scheme
derives cortical dynamics from variational free-energy minimization under an exponential-family
assumption~\cite{kataoka_extended_2026}, a natural sensory-processing counterpart to our
control-side framework. 
Both carry a dually flat manifold, ours pairing the stationary measure
$\nu$ of the planning process with the log-policy under the conditional-negentropy potential $\varphi$, theirs pairing expectation and natural parameters under the log-partition. Whether cortex implements a
sensorimotor duality then becomes the geometric question of what map relates these two dually flat manifolds, which we leave to future work.
\vspace{-2pt}
\section{Conclusion}
In this early-stage work, we demonstrate that the stationary measures (visitation measures) of a resetting planning process form a dually flat statistical manifold,
whose two affine charts are the visitation probabilities and the log-policies, dual under the
conditional entropy. This singles out two optimization algorithms: the natural policy gradient and
policy mirror descent are one update written in the two charts. Planning-as-inference then extends from
linear rewards to nonlinear free energies with each iterate an exact evidence lower bound solved
by one natural-gradient step, and the temporal-difference error acquires an interpretation as a marginal
utility estimate. That the same manifold is simultaneously a variational posterior, a
dynamic-programming object, and a source of unbiased per-cycle return estimates is what lets these
views coincide. Whether neural circuits realize this geometry, and how it extends to continuous
spaces, we leave to future work.

\newpage
\bibliographystyle{plain}
\bibliography{references}

\newpage
\appendix

\section{Details on Planning as Inference}\label{app:planning-as-inference}

\subsection{Restarting visitations reproduce RL}\label{app:reproduce-rl}
The fractional-linear functional
\begin{equation}\label{eq:j-objective}
    J(\pi) \;=\; \frac{\langle r,\nu^\pi\rangle}{\langle \rho,\nu^\pi\rangle}
\end{equation}
reproduces the standard performance criteria of reinforcement learning. To see this, we apply the Charnes--Cooper transformation of the fractional linear program
\begin{equation}
    \sup_{\nu\in\mathcal V_+}\ \frac{\langle r,\nu\rangle}{\langle\rho,\nu\rangle},
    \qquad
    \mathcal V_+=\bigl\{\nu>0:\ K\nu=0,\ \langle\mathbf 1,\nu\rangle=1\bigr\},
\end{equation}
and obtain the standard linear program of Markov decision
processes~\cite{manne1960linear, puterman1994markov}
\begin{equation}
    \sup_{y\in\mathcal Y_+}\ \langle r,y\rangle,
    \qquad
    \mathcal Y_+:=\bigl\{y>0: K y=0,\ \ \langle\rho,y\rangle=1\bigr\},
\end{equation}
where $\rho$ represents generalized discounting~\cite{white2017unifying, jasso2024constrained}. The programs are equivalent in the sense that they share the same optimal value and their
maximizers correspond under the mutually inverse maps
\begin{equation}\label{eq:cc-maps}
    \nu=\frac{y}{\langle\mathbf 1,y\rangle}\ \text{(normalize)},\qquad
    y=\frac{\nu}{\langle\rho,\nu\rangle}\ \text{(unnormalize)}.
\end{equation}
The Charnes-Cooper transform in full would generically add another variable $x$ with additional constraints $\{x:x\geq0,x=\langle\mathbf 1, y\rangle\}$. But because our main constraint $K\nu=0$ happens to be linear without affine part, the second variable does not enter the objective functional and its constraints do not interact with the others, so they are merely a definition and can be dropped.

For constant $\rho\equiv 1-\gamma$ one has $\langle\rho,\nu^\pi\rangle = 1-\gamma$ and
\eqref{eq:j-objective} equals the discounted return
$\mathbb{E}\bigl[\sum_{t\ge 0}\gamma^t r(s_t,a_t)\bigr]$. The clock-augmented choice
$\rho(\tilde s,a)=\mathbf{1}[t=H]$ on states $\tilde s=(s,t)$ with reset to $t=0$
recovers the finite-horizon return. The average-reward (gain) criterion arises in the
vanishing-reset limit $\rho\searrow 0$: as the reset rate vanishes $P_\rho^\pi\to P^\pi$
and $\nu^\pi$ converges to the stationary law of the un-reset chain. The reward
$\langle r,\nu^\pi\rangle$ is the gain, provided $P^\pi$ is unichain so the limit is
unique. Here the criterion is carried by the numerator alone, while the ratio $J$ diverges
and the transformation~\eqref{eq:cc-maps} degenerates as $\langle\rho,\nu\rangle\to 0$.
Finally, general $\rho$ yields state-action-dependent discounting~\cite{white2017unifying},
and thus subsumes recent early-termination objectives from continuous
control~\cite{chane2024cat, milosevic2026stochastic}.

\subsection{Equivalence of stationary and path-space free energies}
\label{app:pi-elbo}
The linear free energy of \S\ref{sec:planning-inference} takes the form reward-minus-divergence
as given. We certify that it is the evidence lower bound of a concrete generative model, and that
$D_\varphi$ is the standard trajectory KL.

A \emph{life} is one renewal cycle of the resetting chain: draw $s_0\sim\mu$. At each step draw
$a_t\sim\pi(\cdot\mid s_t)$, flip a reset coin $R_t\sim\mathrm{Ber}(\rho(s_t,a_t))$, and continue to
$s_{t+1}\sim P(\cdot\mid s_t,a_t)$ if $R_t=0$ or end the life if $R_t=1$. The length $L$ is a.s.\
finite since $\mathbb{E}[L]=\frac{1}{\langle \rho,\nu\rangle}<\infty$ and the process is unichain. Let $p,q$ be the laws of one life under $\pi_{\mathrm{ref}}$
and $\pi$, with per-step preference likelihood $p(O_t{=}1\mid s_t,a_t)\propto\exp(\beta^{-1}r(s_t,a_t))$.
The expected per-life visitation is the (unnormalized) occupancy
$y^\pi(s,a)=\mathbb E_q[\sum_{t<L}\delta_{(s_t,a_t)}(s,a)]=\nu^\pi(s,a)/\langle\rho,\nu^\pi\rangle$~\cite{puterman1994markov,altman1999constrained},
with mean life-length $L^\pi:=\langle\mathbf 1,y^\pi\rangle=1/\langle\rho,\nu^\pi\rangle$.

\paragraph{The two ELBOs.} The path-space evidence is
$Z_{\mathrm{path}}=\mathbb E_p[\exp(\beta^{-1}\sum_{t<L}r)]$, and Jensen against $q$ gives
$\log Z_{\mathrm{path}}\ge\mathbb E_q[\beta^{-1}\sum_{t<L}r]-\KL(q\|p)$. Two identities collapse the
right-hand side onto the stepwise free energy. First, the expected per-life return is the resetting
objective,
\[
  \mathbb E_q\big[\textstyle\sum_{t<L}r\big]=\langle r,y^\pi\rangle
  =\frac{\langle r,\nu^\pi\rangle}{\langle\rho,\nu^\pi\rangle}=J(\pi).
\]
Second, the trajectory KL is the canonical divergence: in $\log(q/p)$ the shared restart density
$\mu(s_0)$, reset coins $\rho$, and transitions $P$ cancel, leaving
$\sum_{t<L}\log[\pi(a_t\mid s_t)/\pi_{\mathrm{ref}}(a_t\mid s_t)]$. Taking $\mathbb E_q$ and using
$y^\pi(s,a)=y^\pi_S(s)\pi(a\mid s)$,
\[
  \KL(q\|p)=\big\langle\KL(\pi\|\pi_{\mathrm{ref}}),y^\pi\big\rangle
  =L^\pi\,D_\varphi(\nu^\pi\|\nu^{\mathrm{ref}}).
\]
Both terms carry the common factor $L^\pi$ (the return is $L^\pi\langle r,\nu^\pi\rangle_{\!*}$ per
life, the KL is $L^\pi D_\varphi$), so the path-space ELBO is exactly the stepwise free energy
scaled by the mean life-length,
\begin{equation}\label{eq:path-vs-step}
  \log Z_{\mathrm{path}}\ \ge\ L^\pi\,\beta^{-1}\mathcal F(\nu^\pi),
  \qquad
  \beta^{-1}\mathcal F(\nu^\pi)=\beta^{-1}\langle r,\nu^\pi\rangle-D_\varphi(\nu^\pi\|\nu^{\mathrm{ref}}).
\end{equation}
At the maximizer $L^\pi$ is a fixed positive scale, so $\log Z_{\mathrm{path}}$ and the stepwise
$\log Z$ of \eqref{eq:logZ} share the same optimal policy; the stepwise bound
\eqref{eq:direct-elbo} is \eqref{eq:path-vs-step} per unit life-length.

\paragraph{The log-partition explicitly.} The first-order conditions of \eqref{eq:logZ} identify
$\log Z$ with the multiplier $c$ of the mass constraint $\langle\mathbf 1,\nu\rangle=1$. Stationarity
gives the tilt $\log(\pi^\star/\pi_{\mathrm{ref}})=\beta^{-1}(r-K^\top V^\star)-c\mathbf 1
=\beta^{-1}A_{\nu^\star}[r]-c\mathbf 1$, where $V^\star$ (the multiplier of $K\nu=0$, in reward
units) is the soft \emph{bias} and $c$ the soft \emph{gain}; per-state normalization
$\sum_a\pi^\star(a\mid s)=1$ turns this into the gain-corrected soft Bellman equation
\begin{equation}\label{eq:soft-bellman-planning}
  V^\star(s)+\beta c
  =\beta\log\sum_a\pi_{\mathrm{ref}}(a\mid s)\exp\!\Big(\tfrac1\beta\big[r(s,a)
  +(1-\rho(s,a))\textstyle\sum_{s'}P(s'\mid s,a)V^\star(s')+\rho(s,a)\langle\mu,V^\star\rangle\big]\Big),
\end{equation}
solvable because the gain $\beta c$ absorbs the constant null direction of the resetting kernel
(cf.\ the gain-corrected Poisson equation of the cotangent space construction in the main text). It fixes $V^\star$ up to the
constant gauge, to which $A_{\nu^\star}[r]=r-K^\top V^\star$ is invariant. Substituting the tilt
back into \eqref{eq:logZ} and using $K\nu^\star=0$, $\langle\mathbf 1,\nu^\star\rangle=1$, the two
reward terms cancel and
\begin{equation}\label{eq:logZ-explicit}
  \log Z=c=\langle\rho,\nu^\star\rangle\,\log Z_{\mathrm{path}}
        =\frac{\log Z_{\mathrm{path}}}{L^{\pi^\star}},
\end{equation}
the soft \emph{gain}: the per-life path log-partition $\log Z_{\mathrm{path}}$ divided by the mean
life-length $L^{\pi^\star}=1/\langle\rho,\nu^\star\rangle$, i.e.\ the free-energy rate of the
renewal process. For constant $\rho\equiv1-\gamma$, $L^{\pi^\star}=1/(1-\gamma)$ and
$c=(1-\gamma)\log Z_{\mathrm{path}}=\beta^{-1}(1-\gamma)\langle\mu,V^\star\rangle$. For constant $\rho\equiv1-\gamma$ the reset branch
$\rho\langle\mu,V^\star\rangle$ is a global constant, \eqref{eq:soft-bellman-planning} reduces to
the ordinary soft Bellman equation, and the gain is the entropy-regularized start value
$c=\beta^{-1}(1-\gamma)\langle\mu,V^\star\rangle$. For state-action-dependent $\rho$ the reset term
couples the backup through the scalar $\langle\mu,V^\star\rangle$.

\section{Details on the Visitation Manifold}

\subsection{Concrete dual charts}
\label{app:mixture-chart}

First, we give the explicit mixture coordinate $\eta$. Because the reset kernel
$P_\rho$ does not depend on the policy, the flow operator $K$ is a fixed linear map, and the
mixture chart is barycentric coordinates on a fixed affine space.

\paragraph{The affine hull.} The point set $\mathcal V_+$ is a relatively open convex subset
of the affine space
\[
  \mathcal H:=\Big\{\nu\in\mathbb R^{\mathcal S\times\mathcal A}:\ K\nu=0,\ \langle\mathbf 1,\nu\rangle=1\Big\},
\]
the intersection of the fixed flow subspace $\ker K$ with the sum-1-hyperplane. Since
$\operatorname{rank}K=|\mathcal S|-1$ (its rows sum to zero) and $\mathbf 1$ is independent of
those rows on $\ker K$, $\mathcal H$ has dimension $d:=|\mathcal S|(|\mathcal A|-1)$, and its
underlying linear space is the tangent space $T=\ker K\cap\ker\mathbf 1^{\!\top}$ of every point
$\nu\in\mathcal V_+$. Note that $\ker\mathbf 1^{\!\top}$ on its own is the annihilator of constant functions and the standard tangent space of the open probability simplex \cite{ay2017information,nagaoka2024fisher}.

\paragraph{Mixture coordinates.} Let $d=|\mathcal S| \times (|\mathcal A|-1)$ and fix $d+1$ affinely independent reference
visitation measures $\nu_0,\dots,\nu_d\in\mathcal H$. Every $\nu\in\mathcal H$
has unique barycentric weights $\lambda_i(\nu)$ with
\[
  \nu=\sum_{i=0}^d\lambda_i(\nu)\,\nu_i,\qquad \sum_{i=0}^d\lambda_i(\nu)=1,
\]
and the \emph{mixture chart} is the reduced weight vector
\begin{equation}\label{eq:mixture-chart}
  \eta(\nu):=\bigl(\lambda_1(\nu),\dots,\lambda_d(\nu)\bigr)\in\mathbb R^d,
\end{equation}
with $\lambda_0=1-\sum_{i\ge1}\lambda_i$ and
\begin{equation}
    \eta^{-1}(x)=\nu_0+\sum_{i=1}^d x_i\,(\nu_i-\nu_0)
\end{equation}
its inverse. These maps give a chart for the \textit{mixture family} of visitation measures and restricted to the point set, $\eta$ is a homeomorphism of $\mathcal V_+$ onto the bounded open convex polytope
$\eta(\mathcal V_+)\subset\mathbb R^d$.

A mixture geodesic is a measure segment
$t\mapsto(1-t)\nu_0'+t\nu_1'$. In the coordinate \eqref{eq:mixture-chart}, it is the affine line
$t\mapsto(1-t)\eta(\nu_0')+t\eta(\nu_1')$, since $\eta$ is affine. Hence its unique affine connection $\nabla^{\mathrm m}$ is
flat in $\eta$ and its Christoffel symbols vanish. $\eta$ is a global $\nabla^{\mathrm m}$-affine
chart, as claimed in the main text. A convenient choice of the reference
measures is $d+1$ affinely independent deterministic-policy occupancies (vertices of the closure of
$\overline{\mathcal V_+}$), which makes the corners interpretable and $\eta^{-1}$ a single
matrix--vector product, the projection illustrated in Figure \ref{fig:geom} on the bottom left.

\paragraph{Exponential coordinates.}
The dual coordinate is the differential of $\varphi$ read in the mixture chart. Since
$\eta\mapsto\nu(\eta)$ is affine \eqref{eq:mixture-chart} with constant Jacobian
$\partial\nu/\partial\eta_i=\nu_i-\nu_0$, the chain rule gives, for the $i$-th dual coordinate,
\begin{equation}\label{eq:dual-chart}
  \theta_i(\nu):=\frac{\partial\varphi}{\partial\eta_i}
  =\Big\langle\frac{\partial\varphi}{\partial\nu},\,\nu_i-\nu_0\Big\rangle
  =\big\langle \log\pi_\nu,\ \nu_i-\nu_0\big\rangle,
\end{equation}
where the last equality uses $\partial\varphi/\partial\nu(s,a)=\log\pi_\nu(a\mid s)$ modulo
state-wise constants: the ambient gradient of the conditional negentropy is the log-policy, and
the state-wise constants annihilate the differences $\nu_i-\nu_0$ (both are flows of unit mass,
so $(\nu_i-\nu_0)$ has zero state marginals and pairs to zero against any state function
$c_s$). Thus $\theta=\nabla\varphi(\eta)$ is a linear image of the log-policy, and $\varphi$
being strictly convex on $\mathcal V_+$ makes $\eta\mapsto\theta$ a diffeomorphism onto its
image; its inverse is $\eta=\nabla\varphi^\ast(\theta)$. In the intrinsic (basis-free) form this
reads $\theta(\nu)=\log\pi_\nu$ modulo state-wise constants---the exponential coordinate is the
log-policy, and the reference-independent representative is fixed by choosing the gauge
$\log\pi_\nu(a_s^\circ\mid s)=0$, giving
$\theta_{s,a}=\log\pi_\nu(a\mid s)-\log\pi_\nu(a_s^\circ\mid s)$.

The two charts therefore have sharply different global shape. The mixture image
$\eta(\mathcal V_+)$ is a \emph{bounded} polytope (an affine image of the bounded set
$\mathcal V_+$), whereas the exponential image
\[
  \theta(\mathcal V_+)=\bigl\{(\log\pi_\nu(a\mid s))_{a\ne a_s^\circ}\bigr\}=\mathbb R^d
\]
is \emph{all} of $\mathbb R^d$: log-probabilities of an open simplex, modulo state-wise
constants, range over the whole space. The Legendre map $\nabla\varphi$ is the homeomorphism
between the bounded m-flat polytope and the complete e-flat space---the finite-dimensional
shadow of a bounded mixture family paired with a complete exponential family. Its geodesics, the
$\theta$-affine lines, are the log-linear policy increments, and $\nabla^{\mathrm e}$ is flat in
$\theta$ by the same argument that made $\nabla^{\mathrm m}$ flat in $\eta$.

\subsection{The generalized Pythagorean theorem certifies dual flatness}
\label{app:pythagoras}

The dual flatness of $\mathcal V_+$ is equivalent to an exact Pythagorean relation for the
Bregman divergence of $\varphi$, and this relation is a directly checkable certificate that the
manifold carries the claimed structure. Throughout, $D_\varphi$ is the canonical divergence of
$\varphi$ and $\theta(\nu)=\nabla\varphi(\nu)=\log\pi_\nu$ the exponential coordinate.

The starting point is an exact three-point identity that holds for \emph{any} convex potential.
Expanding each divergence through $D_\varphi(\nu\|\nu')=\varphi(\nu)-\varphi(\nu')-\langle\theta(\nu'),\nu-\nu'\rangle$
and cancelling the $\varphi(\nu_1),\varphi(\nu_2)$ terms gives, for all
$\nu_1,\nu_2,\nu_3\in\mathcal V_+$,
\begin{equation}\label{eq:three-point}
  D_\varphi(\nu_1\|\nu_2)-D_\varphi(\nu_1\|\nu_3)-D_\varphi(\nu_3\|\nu_2)
  = -\big\langle\,\theta(\nu_2)-\theta(\nu_3),\ \eta(\nu_1)-\eta(\nu_3)\,\big\rangle .
\end{equation}
The content of the correspondence is that the right-hand pairing is the metric inner product of
the two geodesics meeting at $\nu_3$, each read as the chord in its own affine chart. The
mixture geodesic $\nu_3\!\to\!\nu_1$ is $\nu$-affine, so its velocity is the ambient chord
$u_{\mathrm m}=\nu_1-\nu_3\in T_{\nu_3}\mathcal V_+$; the exponential geodesic $\nu_3\!\to\!\nu_2$
is $\theta$-affine, so its velocity $u_{\mathrm e}$ has constant $\theta$-chord
$\theta(\nu_2)-\theta(\nu_3)=\log\pi_{\nu_2}-\log\pi_{\nu_3}$. Since $\theta=\nabla\varphi$ and
$g=\nabla^2\varphi$, the Legendre differential lowers $u_{\mathrm e}$ to this chord,
$\theta(\nu_2)-\theta(\nu_3)=g_{\nu_3}u_{\mathrm e}$. The Euclidean pairing of this chord with
$u_{\mathrm m}$ is then the metric pairing, because the marginal term of $g$ drops against a
mass-zero tangent vector:
\begin{equation}\label{eq:residual-metric}
  \big\langle \log\pi_{\nu_2}-\log\pi_{\nu_3},\,u_{\mathrm m}\big\rangle
  =\sum_{s,a}\frac{u_{\mathrm e}(s,a)\,u_{\mathrm m}(s,a)}{\nu_3(s,a)}
   -\underbrace{\sum_s\frac{u_{\mathrm e,S}(s)\,u_{\mathrm m,S}(s)}{\nu_{3,S}(s)}}_{\text{drops: }\log\pi\text{ is state-centered, }u_{\mathrm m}\in\ker\mathbf 1^{\!\top}}
  =g_{\nu_3}(u_{\mathrm e},u_{\mathrm m}),
\end{equation}
where the first equality uses that $\log\pi_{\nu_2}-\log\pi_{\nu_3}$ is the state-wise-centered
representative of $u_{\mathrm e}/\nu_3$ (log-conditionals carry no marginal component). Hence the
Pythagorean deficit \eqref{eq:three-point} equals $-g_{\nu_3}(u_{\mathrm e},u_{\mathrm m})$,
vanishing exactly at $g$-orthogonality.

\subsection{Linear MDPs and log-linear policies}
\label{app:linear-mdp}

We generalize the dually flat structure to the linear MDP of
Jin et al.~\cite{jin2020provably}, where dynamics and reward are linear in a known feature map
$\phi:\mathcal S\times\mathcal A\to\mathbb R^d$,
\[
  P(s'\mid s,a)=\big\langle\phi(s,a),\,\mathsf m(s')\big\rangle,\qquad
  r(s,a)=\big\langle\phi(s,a),\,\theta_r\big\rangle,
\]
with $\mathsf m=(\mathsf m^{(1)},\dots,\mathsf m^{(d)})$ signed measures on $\mathcal S$
constrained so that $\langle\phi(s,a),\mathsf m\rangle\in\mathcal P(\mathcal S)$ for all $(s,a)$
(i.e.\ $\Phi\mathsf m$ is row-stochastic), and $\theta_r\in\mathbb R^d$.

\paragraph{Both sides collapse to $d$ dimensions.} Writing $\Phi\in\mathbb R^{\mathcal S\mathcal A\times d}$
for the feature matrix and $\psi^\pi:=\Phi^\top\nu^\pi=\mathbb E_{\nu^\pi}[\phi]$ for the
feature-expectation vector, the flow balance of \S\ref{sec:setup} makes the state marginal
affine in $\psi^\pi$,
\begin{equation}\label{eq:linmdp-marginal}
  \nu^\pi_S(s')=(1-\rho)\,\mu(s')+\gamma\,\big\langle\psi^\pi,\,\mathsf m(s')\big\rangle .
\end{equation}
The reward- and dynamics-relevant part of the occupancy is thus the $d$-dimensional linear image
$\mathcal F:=\Phi^\top\mathcal V_+$, a polytope on which the return is linear,
$J(\pi)=\langle\psi^\pi,\theta_r\rangle$, with deterministic-policy vertices. This is the
feature-expectation polytope of apprenticeship learning/feature-matching~\cite{abbeel2004apprenticeship}. Dually,
the conjugate coordinate $\theta=\log\pi_\nu$ restricts to the log-linear class
$\Pi_\phi=\{\pi_w(a\mid s)\propto\exp\langle w,\phi(s,a)\rangle\}$, and $\phi$ is \emph{forced} as
the sufficient statistic: for $\pi_w$ to be Legendre-dual to the mixture family with mean
coordinate $\psi=\mathbb E_\nu[\phi]$, the natural parameter $w$ must pair with $\psi$ through
$\phi$. The metric then pulls back to the $d\times d$ compatible-features Fisher matrix
$G(w)=\mathbb E_{\nu^{\pi_w}}[\operatorname{Cov}_{\pi_w}\phi]$~\cite{kakade2001natural,agarwal2021theory}.
The choice of \emph{conditional} negentropy is what permits this, because it is degenerate along the state-marginal directions.

\paragraph{Exact finite-dimensional PMD.} In a linear MDP $Q^\pi=\langle\phi,\omega^\pi\rangle$
for every policy~\cite{jin2020provably}, so the advantage $A^\pi\in\operatorname{span}\Phi$ and
the exponentiated-advantage step collapses to a $d$-dimensional recursion
\begin{equation}\label{eq:linmdp-recursion}
  w_{k+1}=w_k+\alpha\,\omega^{\pi_k},\qquad G(w_k)^{-1}\nabla_w J(\pi_{w_k})=\omega^{\pi_k},
\end{equation}
which is simultaneously the mirror step \eqref{eq:pmd}, Kakade's natural policy gradient, and
compatible function approximation~\cite{agarwal2021theory} with \emph{zero} approximation error
(the compatible regression is exact because $A^{\pi_k}\in\operatorname{span}\Phi$). More
generally, an exponential-family class $\Pi_T=\{\pi_w\propto\exp\langle w,T\rangle\}$ is closed
under the step \emph{iff} $T$ is \emph{value-complete}, $A^\pi\in\operatorname{span}T$ for all
$\pi\in\Pi_T$. The linear MDP with $T=\phi$ is the canonical case. When value-completeness fails,
\eqref{eq:linmdp-recursion} runs with $\omega^{\pi_k}$ the $T$-projection of $A^{\pi_k}$ and incurs
the compatible-approximation error $\|A^{\pi_k}-\Pi_T A^{\pi_k}\|$, see \cite{agarwal2021theory}.

\paragraph{Towards continuous spaces.} The finite-dimensional collapse above suggests the dual
flatness persists on genuinely continuous $\mathcal S,\mathcal A$ for exponential-family
policies, though a full treatment is beyond our scope. We record the expectation as a conjecture.

\begin{conjecture}\label{conj:continuous}
Let $\mathcal S,\mathcal A$ be standard Borel and $\Pi_T=\{\pi_w\propto\exp\langle w,T\rangle\}$ an
exponential-family policy class with sufficient statistic $T\in L^2$ and full-support restart law
$\mu$. Under the regularity conditions of parametrized measure models~\cite{ay2017information}
($k$-integrability of $T$, $\sigma$-finiteness of the occupancy) and with the conditional
negentropy taken as Bregman potential in the directional-derivative sense of measure-space mirror
descent~\cite{aubin2022mirror}, the resetting visitation measures
$\{\nu^{\pi_w}:w\in\mathbb R^{d'}\}$ form a finite-dimensional dually flat submanifold of
$\mathcal P(\mathcal S\times\mathcal A)$, with $w\mapsto\log\pi_w$ and
$w\mapsto\mathbb E_{\nu^{\pi_w}}[T]$ the Legendre-dual coordinates and the compatible-features
Fisher matrix as metric. In particular, the finite-dimensionality is inherited from $T$, not from
$|\mathcal S||\mathcal A|$, and the mirror/NPG equivalence of \S\ref{sec:md-npg} holds verbatim in
the $w$-coordinates.
\end{conjecture}
The two ingredients are independently established: finite-dimensional exponential families are
dually flat submanifolds of the infinite-dimensional finite probability measure space~\cite{amari2000methods,ay2017information},
and mirror descent with the entropy potential converges over measure spaces under relative
smoothness~\cite{aubin2022mirror}. So what the conjecture asserts is that the
\emph{resetting flow constraint} is compatible with both, i.e.\ that $\nu^{\pi_w}$ stays in the
common domain where the directional-derivative Bregman geometry and the exponential-family dual
flatness coincide.

\section{Relative smoothness}\label{app:relsmooth}
\begin{proposition}[Relative smoothness]\label{prop:relsmooth-airtight}
$F=\langle r,\nu\rangle+\tau H(\nu_S)$ is $L_F$-smooth relative to $\varphi$ on $\mathcal V_+$,
i.e.\ $-\nabla^2F(\nu)\preceq L_F\,g_\nu$ on $T_\nu\mathcal V_+$ for every
$\nu\in\operatorname{relint}\mathcal V_+$, with
\begin{equation}
    \boxed{\,L_F=\frac{\tau\,\gamma^2\,C}{(1-\gamma)^2}\,},\qquad
    C:=\sup_{\nu\in\mathcal V_+}\Bigl\|\tfrac{\nu_1}{\mu}\Bigr\|_\infty\le\frac1{\mu_{\min}},
    \quad \nu_1:=P_*\nu .
\end{equation}
\end{proposition}

\begin{proof}
By the second-order characterization of relative smoothness~\cite[Prop.~1.1]{lu2018relatively},
it suffices to show $-\nabla^2F(\nu)(u,u)\le L_F\,g_\nu(u,u)$ for every
$\nu\in\mathcal V_+$ and $u\in T_\nu\mathcal V_+=\ker K\cap\ker\mathbf 1^{\!\top}$.

\emph{Two sides of the Hessian bound.} The linear term of $F$ drops and
the Hessian of $H(\nu_S)$ is
$-\nabla^2F(\nu)(u,u)=\tau\|u_S\|_{1/\nu_S}^2$ the Fisher form on state-marginals. For the right side of the inequality, the tangent decomposition
\eqref{eq:tangent-decomp} $u=u_S\pi_\nu+\nu_S\dot \pi_\nu$ substituted into
$g_\nu(u,u)=\sum_{s,a}u^2/\nu-\sum_s u_S^2/\nu_S$ cancels the marginal term and gives
$g_\nu(u,u)=\|\sigma\|_{1/\nu}^2$ with $\sigma:=\nu_S\dot \pi_\nu$, $\sum_a\sigma(s,a)=0$. It remains
to bound $\|u_S\|_{1/\nu_S}^2$ by $\|\sigma\|_{1/\nu}^2$.

\emph{Marginal determined by the conditional (using $u\in T_\nu\mathcal V_+$).}
Differentiating the flow identity $\nu_S=(1-\gamma)\mu+\gamma P_*\nu$ along $u$ gives
$(I-\gamma(P^\pi)^{\!\top})u_S=\gamma P_*\sigma$, whence
\begin{equation}\label{eq:marg-id}
    u_S=\tfrac{\gamma}{1-\gamma}\,R\,(P_*\sigma),\qquad
    R:=(1-\gamma)(I-\gamma(P^\pi)^{\!\top})^{-1},
\end{equation}
where $R$ is column-stochastic with $R\mu=\nu_S$ (the discounted state occupancy), and
$P_*$ the state pushforward. Without flow balance $u_S$ would be free but \eqref{eq:marg-id} ties it to $\sigma$.
Both $P_*$ and $R$ are Markov (row-/column-stochastic)
operators, so each is a $\chi^2$-contraction between the corresponding weighted norms (the
data-processing inequality, cf.~\cite{ay2017information}):
$\|P_*\sigma\|_{1/\nu_1}^2\le\|\sigma\|_{1/\nu}^2$ and
$\|R w\|_{1/\nu_S}^2\le\|w\|_{1/\mu}^2$. The single non-contraction is the reference change
$\|w\|_{1/\mu}^2\le\|\nu_1/\mu\|_\infty\,\|w\|_{1/\nu_1}^2$, contributing the concentrability
$C$ (uniform since each $\nu_1$ is a probability vector, so $C\le1/\mu_{\min}$). Combining,
\[
    \|u_S\|_{1/\nu_S}^2
    =\Bigl(\tfrac{\gamma}{1-\gamma}\Bigr)^2\|R P_*\sigma\|_{1/\nu_S}^2
    \le\Bigl(\tfrac{\gamma}{1-\gamma}\Bigr)^2 C\,\|\sigma\|_{1/\nu}^2,
\]
which with the two sides above is the Hessian domination with $L_F=\tau\gamma^2C/(1-\gamma)^2$.
\end{proof}

\begin{corollary}[Relative strong convexity fails]\label{cor:no-rsc}
No $\eta>0$ satisfies $-\nabla^2F(\nu)\succeq\eta\,g_\nu$ on $T_\nu\mathcal V_+$: at any state $s'$
with two actions sharing a transition ($P(\cdot|s',a_1)=P(\cdot|s',a_2)$), the variation
$\dot \pi_\nu(a_1|s')=+1,\dot \pi_\nu(a_2|s')=-1$ has $P_*\sigma=0$, hence $u_S=0$ and
$-\nabla^2F(\nu)(u,u)=0$ while $g_\nu(u,u)>0$.
\end{corollary}


\end{document}